\documentclass{article}

\usepackage[nonatbib,preprint]{neurips_2022}

\usepackage[utf8]{inputenc} 
\usepackage[T1]{fontenc}    
\usepackage{hyperref}       
\usepackage{url}            
\usepackage{booktabs}       
\usepackage{amsfonts}       
\usepackage{nicefrac}       
\usepackage{microtype}      
\usepackage{xcolor}         
\usepackage{amsmath}
\usepackage{graphicx}
\usepackage{caption}
\usepackage{subcaption}
\usepackage{tabularx}
\usepackage{multirow}

\newcommand{\toolname}{CertiFair}

\usepackage{wrapfig}
\usepackage{amsthm}
 \newtheorem{theorem}{\textbf{Theorem}}[section]

 \newtheorem{proposition}[theorem]{Proposition}
 
 \newtheorem{definition}[theorem]{Definition}

\title{
CertiFair: A Framework for Certified Global Fairness of Neural Networks
}

\author{%
   Haitham Khedr and Yasser Shoukry\\
  Department of Electrical Engineering and Computer Science\\
  University of California, Irvine \\
  Irvine, CA 92697, USA \\
  \texttt{\{hkhedr, yshoukry\}@uci.edu} 
}

\begin{document}

\maketitle

\begin{abstract}
We consider the problem of whether a Neural Network (NN) model satisfies global individual fairness. Individual Fairness (defined in \cite{dwork2012fairness}) suggests that \textit{similar} individuals with respect to a certain task are to be treated \textit{similarly} by the decision model. In this work, we have two main objectives. The first is to construct a verifier which checks whether the fairness property holds for a given NN in a classification task or provide a counterexample if it is violated, i.e., the model is fair if all similar individuals are classified the same, and unfair if a pair of similar individuals are classified differently. To that end, We construct a sound and complete verifier that verifies global individual fairness properties of ReLU NN classifiers using distance-based similarity metrics. The second objective of this paper is to provide a method for training provably fair NN classifiers from unfair (biased) data. We propose a fairness loss that can be used during training to enforce fair outcomes for similar individuals. We then provide provable bounds on the fairness of the resulting NN. We run experiments on commonly used fairness datasets that are publicly available and we show that global individual fairness can be improved by 96 \% without significant drop in test accuracy.
\end{abstract}

\section{Introduction}
\label{sec:intro}

Neural Networks (NNs) have become an increasingly central component of modern decision-making systems, including those that are used in sensitive/legal domains such as crime prediction \cite{brennan2009evaluating}, credit assessment \cite{Dua2019}, income prediction \cite{Dua2019}, and hiring decisions. However, studies have shown that these systems are prone to biases \cite{mehrabi2021survey} that deem their usage \textit{unfair} to unprivileged users based on their age, race, or gender. The bias is usually either inherent in the training data or introduced during the training process. Mitigating algorithmic bias has been studied in the literature ~\cite{zhang2018mitigating,xu2018fairgan,mehrabi2021attributing} in the context of group and individual fairness. However, the fairness of the NN is considered only \textit{empirically} on the test data with the hope that it represents the underlying data distribution. 

Unlike the \emph{empirical} techniques for fairness, we are interested to provide \emph{provable certificates} regarding the fairness of a NN classifier. In particular, we focus on the ``global individual fairness'' property which states that a NN classification model is globally individually fair if \emph{all} \emph{similar} pairs of inputs $x$, $x'$ are assigned the same class. We use a feature-wise closeness metric instead of an $\ell_p$ norm to evaluate similarity between individuals, i.e, a pair $x$, $x'$ is similar if for all features $i$, $|x_i - x'_i| \leq \delta_i$. 
Given this fairness notion, the objective of this paper is twofold. First, it aims to provide a sound and complete formal verification framework that can automatically certify whether a NN satisfy the fairness property or produce a concrete counterexample showing two inputs that are not treated fairly by the NN. Second, this paper provides a training procedure for certified fair training of NNs even when the training data is biased.

\paragraph{Challenge:} Several existing techniques focus on generalizing ideas from \emph{adversarial robustness} to reason about NN fairness~\cite{yurochkin2019training,ruoss2020learning}. By viewing unfairness as an adversarial noise that can flip the output of a classifier, these techniques can certify the fairness of a NN \emph{locally}, i.e., in the neighborhood of a given individual input. In contrast, this paper focuses on \emph{global} fairness properties where the goal is to ensure that the NN is fair with respect to \emph{all} the similar inputs in its domain. Such a fundamental difference precludes the use of existing techniques from the literature on adversarial robustness and calls for novel techniques that can provide provable fairness guarantees.


\paragraph{This work:} We introduce \toolname, a framework for certified global fairness of NNs. This framework consists of two components. First, a verifier that can prove whether the NN satisfies the fairness property or produce a concrete counterexample that violates the fairness property. This verifier is motivated by the recent results in the ``relational verification'' problem \cite{barthe2011relational} where the goal is to verify \emph{hyperproperties} that are defined over pairs of program traces. Our approach is based on the observation that the global individual fairness property~\eqref{eq:fairness} can be seen as a \emph{hyperproperty} and hence we can generalize the concept of \emph{product programs} to \emph{product NNs} that accepts a pair of inputs $(x,x')$, instead of a single input $x$, and generates two independent outputs for each input. A global fairness property for this product NN can then be verified using existing NN verifiers. Moreover, and inspired by methods in certified robustness, we also propose a training procedure for \emph{certified fairness} of NNs. Thanks again to the product NN, mentioned above, one can establish upper bounds on fairness and use it as a regularizer during the training of NNs. Such a regularizer will promote the fairness of the resulting model, even if the data used for training is biased and can lead to an unfair classifier. While such \emph{fairness regularizer} will enhance the fairness of the model, one needs to check if the fairness property holds globally using the sound and complete verifier mentioned above.



\paragraph{Contributions:} Our main contributions are:
\begin{itemize}
    \item To the best of our knowledge, we present the first sound and complete NN verifier  for global individual fairness properties.
    \item A method for training NN classifiers with a modified loss that enforces fair outcomes for similar individuals. We provide bounds on the loss in fairness which constructs a certificate on the fairness of the trained NN.
    \item We applied our framework to common fairness datasets and we show that global fairness can be achieved with a minimal loss in performance.
\end{itemize}

\section{Preliminaries}
Our framework supports regression and multi-class classification models, however for simplicity, we only present our framework for binary classification models $h: \mathbb{R}^n \rightarrow  \{0,1\}$ of the form $h(x)=t(f_\theta(x))$ where $t$ is a threshold function with threshold equals to $0.5$. Moreover, we assume $f_\theta$ is an $L$-layer NN with ReLU hidden activations and parameters $\theta = \left( (W_1, b_1), \ldots, (W_L, b_L)\right)$ where $(W_i,b_i)$ denotes the weights and bias of the $i$th layer.
We also assume the activation function of the last layer of $f_\theta$ is a sigmoid function. The NN accepts an input vector $x$ where the components $x_i \in \mathbb{R}$ (the set of real numbers) $\text{ or } x_i \in\mathbb{Z}$ (the set of integer numbers). This is suitable for most of the datasets where some features of the input are numerical while others are categorical. In this paper, we are interested in the following fairness property:

\begin{definition}[Global Individual Fairness~\cite{dwork2012fairness, john2020verifying}]
A model $f_\theta(x)$ is said to satisfy the global individual fairness property $\phi$ if the following holds:
\begin{align}
\forall x,x' \in D_\phi \quad \text{s.t.} \quad d(x,x') = 1 \Longrightarrow h(f_\theta(x)) = h(f_\theta(x')), \label{eq:fairness}
\end{align}
where $d: \mathcal{R}^n \times \mathcal{R}^n \rightarrow \{1,0\}$ is a similarity metric that evaluates to $1$ when $x$ and $x'$ are similar and $D_\phi$ is the input domain of $x$ for property $\phi$ defined as $D_\phi:= D_\phi^0 \times ... \times D_\phi^{n-1}$ with $D_\phi^i:= \{x_i \; | \; l_i \leq x_i \leq u_i\}$ for some bounds $l_i$ and $u_i$.
\label{def:global_fairness}
\end{definition}

In this paper, we utilize the feature-wise similarity metric $d$ defined as:
\begin{align}
 d(x,x') = \begin{cases} 
1 & \text{if} \; |x_i - x'_i| \le \delta_i \qquad \forall i \in \{1,\ldots n\} \\
0 & \text{otherwise}
\end{cases} \label{eq:similalrity}
\end{align}
This feature-wise similarity metric allows the fairness property $\phi$ in~\eqref{eq:fairness} to capture several other fairness properties as special cases as follows:
\begin{definition}[Individual discrimination \cite{aggarwal2019}]
\label{def:counterfactual}
A model $f_\theta(x)$ is said to be nondiscriminatory between individuals if the following holds:
$$ 
\forall x = (x_s, x_{ns}),x' =(x'_s, x'_{ns}) \in D_\phi \quad \text{s.t.} \quad  x_{ns} = x'_{ns} \; \text{and} \; x_s \ne x'_{s} \Longrightarrow h(f_\theta(x)) = h(f_\theta(x')),
$$
where $x_s$ and $x_{ns}$ denotes the sensitive attributes and non-sensitive attributes of $x$, respectively.
\end{definition}
Indeed, the individual discrimination corresponds to a global individual fairness property by setting $\delta_i = 0$ in~\eqref{eq:similalrity} for the non-sensitive attributes. Another definition of fairness\cite{ruoss2020learning} states that two individuals are similar if their numerical features differ by no more than $\alpha$. Again, this can be represented by the closeness metric simply by setting $\delta_i=0$ for categorical attributes and $\delta_i = \alpha$ for numerical attributes.

Based on Definition 1 above, we can formally verify whether the fairness property $\phi$ holds by checking if the set of counterexamples (or violations) $\mathcal{C}$ is empty, where $\mathcal{C}$ is defined as:
\begin{equation}
    \mathcal{C} = \left\{(x,x') \Big\lvert x,x' \in D_\phi,\overset{n-1}{\underset{i=0}{\bigwedge}}|x_i -x_i'| < \delta_i , h(x) \neq h(x')\right\} = \emptyset\text{.}
    \label{eq:verification_pb}
\end{equation}

\section{Global Individual Fairness as a hyperproperty}
\label{hyperproperties}
In this section, we draw the connection between the verification of global individual fairness properties~\eqref{eq:fairness} and hyperproperties in the context of program verification. On the one hand, several local properties of NNs (e.g., adversarial robustness) are considered trace properties, i.e., properties defined on the input-output behavior of the NN. In this case, one can search the input space of the NN to find a \emph{single} input (or counterexample) that leads to an output that violates the property. In the domain of adversarial robustness, a counterexample corresponds to a disturbance to an input that can change the classification output of a NN. On the other hand, other properties, like the global fairness properties, can not be modeled as trace properties. This stems from the fact that one can not judge the correctness of the NN by considering individual inputs to the NN. Instead, finding a counterexample to the fairness property will entail searching over \emph{pairs} of inputs and comparing the NN outputs of these inputs. Properties that require examining pairs or sets of traces (input-outputs of a program) are defined as \emph{hyperproperties}~\cite{barthe2011relational}.

\begin{wrapfigure}{r}{0.45\textwidth}
  \vspace{-9mm}
  \centering
  \includegraphics[width=0.45\textwidth,keepaspectratio]{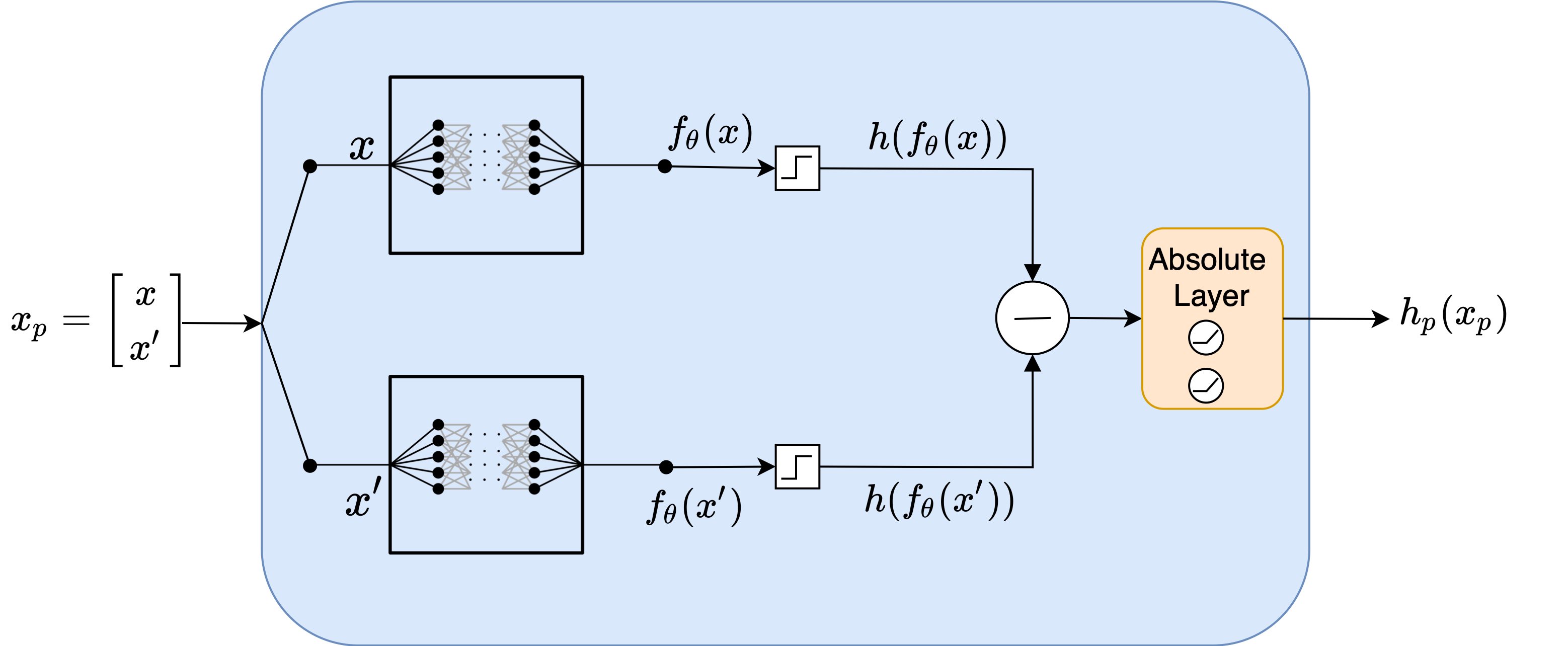}\vspace{-1mm}
  \caption{Construction of the Product NN.
  }
  \label{fig:product_net}
\end{wrapfigure}
Modeling global individual fairness as a hyperproperty leads to a direct certification framework. In particular, a key idea in the hyperproperty verification literature is the notion of a product program that allows the reduction of the hyperproperty verification problem to a standard verification problem~\cite{barthe2011relational}. A product program is constructed by composing two copies of the original program together. The main benefit is that the hyperproperties of the original program become trace properties of the product program that can be verified using standard techniques. Motivated by this observation, our framework \toolname\ generalizes the concept of \emph{product programs} into \emph{product NNs} (described in detail in Section \ref{sec:fair_ver} and shown in Figure~\ref{fig:product_net}) that accepts a pair of inputs and generates a pair of two independent outputs. We then use the product network to verify fairness (hyper)properties using standard techniques.


\section{\toolname: A Framework for Certified Fairness of Neural Networks}
\label{sec:framework}

As mentioned earlier in section \ref{hyperproperties}, the fairness property can be viewed as a hyperproperty of the NN. We propose the use of a product NN that can reduce the verification of such hyperproperty into standard trace (input/output) property. In this section, we first explain how to construct the product NN followed by how to use it to encode the fairness verification problem into ones that are accepted by off-the-shelf NN verifiers. Next, we discuss how to use this product NN to derive a fairness regualrizer that can be used during training to obtain a certified fair NN.

\subsection{Product Neural Network}
\label{sec:prdouct_NN}
Given a neural network $f_\theta$, the product network $f_{\theta_p}$ is basically a side-to-side composition of $f_\theta$ with itself. More formally, the parameters vector $\theta_p$ of the product NN is defined as:
\begin{align}
\theta_p = \left( 
\left( 
\begin{bmatrix}
 W_1 & \mathbf{0} \\
 \mathbf{0} & W_1
\end{bmatrix}
,
\begin{bmatrix}
b_1\\
b_1
\end{bmatrix}
\right),
\ldots,
\left( 
\begin{bmatrix}
 W_L & \mathbf{0} \\
 \mathbf{0} & W_L
\end{bmatrix}
,
\begin{bmatrix}
b_L\\
b_L
\end{bmatrix}
\right)
\right)
\end{align} 
where $(W_i, b_i)$ are the weights and biases of the $i$th layer of $f_\theta$. The input to the product network $f_{\theta_p}$ is a pair of concatenated inputs $x_p=(x,x')$. Finally, we add an output layer that results in an output $h_p \in \{0, 1\}$ defined as:
$ h_p(f_{\theta_p} (x_p)) = |h(f_\theta(x)) - h(f_\theta(x'))| $
where the absolute value operator $|.|$ can be implemented using ReLU nodes by noticing that $|a| = max(a,0) + max(-a,0)$. Figure \ref{fig:product_net} summarizes this construction. 


\subsection{Fairness Verification}
\label{sec:fair_ver}

Using the product network defined above, we can rewrite the set of counterexamples $\mathcal{C}$ in~\eqref{eq:verification_pb} as:
\begin{align}
    \mathcal{C}_p = \left\{x_p \Big\lvert x_p \in D_\phi \times D_\phi , \; \overset{n-1}{\underset{i=0}{\bigwedge}}|x_i -x_i'| < \delta_i, \; h_p(x_p) > 0 \right\} 
\label{eq:verification_pb_p}    
\end{align}
which corresponds to the standard verification of NN input-output properties~\cite{LiuAlgorithmsVerifyingDeep2019}, albeit being defined over the product network inputs and outputs..



To check the emptiness of the set $\mathcal{C}_p$ in~\eqref{eq:verification_pb_p} (and hence certify the global individual fairness property), we need to search the space $D_\phi \times D_\phi$ to find at least one counterexample that violates the fairness property, i.e., a pair $x_p = (x,x')$ that represent similar individuals who are classified differently by the NN. Finding such a counterexample is, in general, NP-hard due to the non-convexity of the ReLU NN $f_{\theta_p}$. To that end, we use PeregriNN \cite{khedr2021peregrinn} as our NN verifier. Briefly, PeregriNN overapproximates the highly non-convex NN  with a linear convex relaxation for each ReLU activation. This is done by introducing two optimization variables for each ReLU, a pre-activation variable $\hat{y}$ and a post-activation variable $y$. The non-convex ReLU function can then be overapproximated by a triangular region of three constraints; $y \geq 0$, $y \geq \hat{y}$, and $y \leq \frac{u}{u-l}(\hat{y} - l)$, where $l\text{,}u$ are the lower and upper bounds of $\hat{y}$ respectively. The solver tries to check whether the approximate problem has no solution or iteratively refine the NN approximation until a counterexample that violates the fairness constraint is found. PeregriNN employs other optimizations in the objective function to guide the refinement of the NN approximation but the details of these methods are beyond the scope of this paper. We refer the reader to the original paper \cite{khedr2021peregrinn} for more details on the internals of the solver. 

\begin{proposition}
Consider a NN model $f_\theta$ and a fairness property $\phi$---either representing a Global Individual Fairness property (Definition~\ref{def:global_fairness}) or an Individual Discrimination property (Definition~\ref{def:counterfactual}). Consider a set of counterexamples $\mathcal{C}_p$ computed using a NN verifier applied to the product network $f_{\theta_p}$. The NN satisfies the property $\phi$ whenever the set $\mathcal{C}_p$ is empty.
\end{proposition}
\begin{proof}
This result follows directly from the equivalence between the sets $\mathcal{C}$ in~\eqref{eq:verification_pb} and $\mathcal{C}_p$ in~\eqref{eq:verification_pb_p} along with the NN verifiers (like PeregriNN) being sound and complete and hence capable of finding any counterexample, if one exists.
\end{proof}

\subsection{Certified Fair Training}
\label{subsec:fair_training}
In this section, we formalize a \emph{fairness regularizer} that can be used to train certified fair models. In particular, we propose two fairness regularizers that correspond to \emph{local} and \emph{global} individual fairness. We provide the formal definitions of both these regularizers below and their characteristics.

\paragraph{Local Fairness Regularizer $\mathcal{L}_f^l$ using Robustness around Training Data:}
Our first proposed regularizer is motivated by the \emph{robustness} regularizers used in the literature of certified robustness~\cite{wong2017provable,zhang2019towards}. The regualrizer, denoted by $\mathcal{L}_f^l$, aims to minimize the average loss in fairness across all training data. More formally, given a training point $(x,y)$ and NN parameters $\theta$, let $\mathcal{L}(f_\theta(x), y; \theta) = -[y \log(f_\theta(x)) + (1-y) \log(1-f_\theta(x))]$ be the standard binary cross-entropy loss. The fairness regularizer $\mathcal{L}_f^l$ can then be defined as:
\begin{align}
\mathcal{L}_f^l(\theta) &= 
\mathbb{E}_{(x,y) \in(X,Y)} \left[ \max_{ \substack{x'\in D_\phi \\ d(x,x') = 1 } } \mathcal{L}(f_{\theta}(x'),y;\theta) \right] 
\label{eq:reg_local_1}
\end{align}
In other words, the regularizer above aims to measure the expected value (across the training data) for the worst-case loss of fairness due to points $x'$ that are assigned to different classes. Indeed, the regualrizer~\eqref{eq:reg_local_1} is not differentiable (with respect to the weights $\theta$) due to the existence of the max operator. Nevertheless, one can compute an upper bound of~\eqref{eq:reg_local_1} and aims to minimize this upper bound instead. Such upper bound can be derived as follows:
\begin{align}
\max_{ \substack{x'\in D_\phi \\ d(x,x') = 1 } } \!\!\! \mathcal{L}(f_\theta(x'),y;\theta) 
&= \!\!\!\!\!
\max_{ \substack{x'\in D_\phi \\ d(x,x') = 1 } } \!\!\!
\begin{cases}
    -\log(1 - f_{\theta}(x')) & \text{if} \; y \!=\! 0 \\
    -\log(f_\theta(x')) & \text{if} \; y \!=\! 1 \\
    \end{cases} \le 
\begin{cases}
    -\log(1 - \theta^T \overline{w}_{S_\phi(x)}) & \text{if} \; y \!=\! 0 \\
    -\log(\theta^T \underline{w}_{S_\phi(x)}) & \text{if} \; y \!=\! 1 \\
    \end{cases}
    \label{eq:reg_local_2}
\end{align}
where $\theta^T \overline{w}_{S_\phi(x)}$ and $\theta^T \underline{w}_{S_\phi}$ are the linear upper/lower bound of $f_{\theta}(x')$ inside the set $S_\phi(x) = \{x'\in D_\phi | d(x,x') = 1\}$. Such linear upper/lower bound of $f_{\theta}(x')$ can be computed using off-the-shelf bounding techniques like Symbolic Interval Analysis~\cite{wang2018efficient} and $\alpha$-Crown~\cite{xu2020fast}. We denote by $\overline{\mathcal{L}}(y; \theta)$ the right hand side of the inequality in~\eqref{eq:reg_local_2} which depends only on the label $y$ and the NN parameters $\theta$. Now the fairness property can be incorporated in training by optimizing the following problem over $\theta$ (the NN parameters):
\begin{equation}
    \min_{\theta} \mathop{\mathbb{E}}_{(x,y) \in (X,y)} \left[(1-\lambda_f)\underbrace{\mathcal{L}(f_\theta(x),y;\theta)}_{\text{natural loss}} + 
    \lambda_f \underbrace{\overline{\mathcal{L}}(y; \theta)}_{\text{local fairness loss}} \right]\text{,}
    \label{train_loss}
\end{equation}
where $\lambda_f$ is a regularization parameter to control the trade-off between fairness and accuracy.

Although easy to compute and incorporate in training, the regularizer $\mathcal{L}_f^l(\theta)$ (and its upper bound) defined above suffers from a significant drawback. It focuses on the fairness \emph{around} the samples presented in the training data. In other words, although the aim was to promote \emph{global} fairness, this regularizer is effectively penalizing the training only in the \emph{local} neighborhood of the training data. Therefore, its effectiveness depends greatly on the quality of the training data and its distribution. Poor data distribution may lead to the poor effect of this regularizer. Next, we introduce another regularizer that avoids such problems.

\paragraph{Global Fairness Regularizer $\mathcal{L}_f^g$ using Product Network:}
To avoid the dependency on data, we introduce a novel fairness regularizer capable of capturing global fairness during the training. Such a regularizer is made possible thanks to the product NN defined above. In particular, the global fairness regularizer $\mathcal{L}_f^g(\theta)$ is defined as:
\begin{align}
    \mathcal{L}_f^g(\theta)
    &= \max_{ \substack{(x,x') \in D_{\phi} \times D_{\phi} \\ d(x,x') = 1}} |f_\theta(x) - f_\theta(x')| 
    \label{eq:reg_global_1}
\end{align}
In other words, the regualizer $\mathcal{L}_f^g(\theta)$ in~\eqref{eq:reg_global_1} aims to penalize the worst case loss in global fairness. Similar to~\eqref{eq:reg_local_1}, the $\mathcal{L}_f^g(\theta)$ is also non-differentiable with respect to $\theta$. Nevertheless, and thanks to the product NN, we can upper bound $\mathcal{L}_f^g(\theta)$ as:
\begin{align}
    \mathcal{L}_f^g(\theta) 
    \le  \max_{(x,x') \in D_{\phi} \times D_{\phi}} |f_\theta(x) - f_\theta(x')| 
    = \max_{ \substack{x_p \in D_{\phi} \times D_{\phi}}} f_{p}(x_p) \le \theta^T \overline{w}_{D_\phi}
    \label{eq:reg_global_2}
\end{align}
where $\theta^T \overline{w}_{D_\phi}$ is the linear upper bound of the product network among the domain $D_\phi \times D_\phi$. Again, such bound can be computed using Symbolic Interval Analysis and $\alpha$-Crown on the product network after replacing the output $h_p$ with $f_p = |f_\theta(x) - f_\theta(x')|$. It is crucial to note that the upper bound in~\eqref{eq:reg_global_2} depends only on the domain $D_\phi$. Hence, the fairness property can now be incorporated into training by minimizing this upper bound as:
\begin{equation}
    \min_{\theta} \mathop{\mathbb{E}}_{(x,y) \in (X,y)} \left[(1-\lambda_f)\underbrace{\mathcal{L}(f_\theta(x),y;\theta)}_{\text{natural loss}} \right] + 
    \lambda_f
    \underbrace{\theta^T \overline{w}_{D_\phi}}_{\text{global fairness loss}} \text{,}
    \label{train_loss}
\end{equation}
where the fairness loss is now outside the $\mathbb{E}[\;.\;]$ operator.

In the next section, we show that the global fairness regularizer $\mathcal{L}_f^g(\theta)$ empirically outperforms the local fairness regularizer $\mathcal{L}_f^l(\theta)$. We end up our discussion in this section with the following result:
\begin{proposition}
Consider a NN model $f_\theta$ and a fairness property $\phi$---either representing a Global Individual Fairness property (Definition~\ref{def:global_fairness}) or an Individual Discrimination property (Definition~\ref{def:counterfactual}). Consider a NN model $f_\theta$ trained using the objective function in~\eqref{train_loss}. If $\theta^T \overline{w}_{D_\phi} = 0$ by the end of the training, then the resulting $f_\theta$ is guaranteed to satisfy $\phi$. 
\end{proposition}
\begin{proof}
The result follows directly from equation~\eqref{eq:reg_global_2}. 
\end{proof}
Indeed, the result above is just a sufficient condition. In other words, the NN may still satisfy the fairness property $\phi$ even if $\theta^T \overline{w}_{D_\phi} > 0$. Such cases can be handled by applying the verification procedure in Section~\ref{sec:fair_ver} after training the NN.

\section{Experimental evaluation}
\label{sec:exp}

We present an experimental evaluation to study the effect of our proposed fairness regularizers and hyperparameters on the global fairness.  We evaluated \toolname\ on four widely investigated fairness datasets (Adult \cite{Dua2019}, German \cite{Dua2019}, Compas \cite{Angwin2013}, and Law School \cite{wightman1998lsac}). All datasets were pre-processed such that any missing rows or columns were dropped, features were scaled so that they're between $[0,1]$, and categorical features were one-hot encoded.

\textbf{Implementation:} We implemented our framework in a Python tool called \toolname\ that can be used for training and verification of NNs against an individual fairness property. \toolname\ uses Pytorch for all NN training and a publicly available implementation of PeregriNN \cite{khedr2021peregrinn} as a NN verifier. We run all our experiments using a single GeForce RTX 2080 Ti GPU and two 24-core Intel(R) Xeon(R) CPU E5-2650 v4 @ 2.20GHz (only 8 cores were used for these experiments). 


\begin{wraptable}[14]{r}{0.3\textwidth}
    \vspace{-5mm}
    \caption{Comparison between verifying local and global fairness properties on the Adult dataset for different similarity constraint (distance $\delta_i$).}
    \centering
    \resizebox{0.3\textwidth}{!}{%
    \vspace{-5mm}
    \begin{tabular}{c|cc}
        \toprule
        $\delta_i$ & 
        Certified & Certified\\ 
        & local & global \\
        & fairness & fairness \\
        & (\%) & (\%) \\ \hline
        0.02	&89.25	&6.56  \\
        0.03	&100.00	&65.98  \\
        0.05	&81.42	&6.40  \\
        0.07	&100.00    &57.39 \\
        0.1  & 99.95  & 66.35\\
        \bottomrule
    \end{tabular}
    }
    \label{tab:exp1}
\end{wraptable}

\textbf{Measuring global fairness using verification:} While the verifier (described in Section~\ref{sec:fair_ver}) is capable of finding concrete counterexamples that violate the fairness, it is also important to \emph{quantify} a bound on the fairness. In these experiments, the certified global fairness is quantified as the percentage of partitions of the input space with zero counterexamples. In particular, the input space is partitioned using the categorical features, i.e., the number of partitions is equal to the number of different categorical assignments and each partition corresponds to one categorical assignment. Note that the numerical features don't have to be fixed inside each partition (property dependant). To verify the property globally, we run the verifier on each partition of the input space and verify the fairness property. Finally, we count the number of verified fair partitions (normalized by the total number of partitions). 

\textbf{Fairness properties:} In the experimental evaluation, we consider two classes of fairness properties. The first class $\mathcal{P}_1$ is the one in definition \ref{def:counterfactual} where two individuals are similar if they only differ in their sensitive attribute. The second class of properties $\mathcal{P}_2$ relaxes the first by also allowing numerical attributes to be close (not identical), this is allowed under definition \ref{def:global_fairness} of global individual fairness by setting $\delta_i >0$ for numerical attributes. 


\subsection{Experiment 1: Global Individual Fairness vs. Local Individual Fairness}
In this experiment, we empirically show that NNs with high \emph{local} individual fairness does not necessarily result into NNs with \emph{global} individual fairness. In particular, we train a NN on the Adult dataset and considered multiple fairness properties (all from class $\mathcal{P}_2$ defined above) by varying $\delta_i$. Note that $\delta_i$ is equal for all features $i$ within the same property, but is different from one property to another. Next, we use PeregeriNN verifier to find counterexamples for both the \emph{local} fairness (by applying the verifier to the trained NN) and the \emph{global} fairness (by applying the verifier to the product NN). We measure the fairness of the NN for both cases and report the results in Table~\ref{tab:exp1}. The results indicate that verifying \emph{local} fairness may result in incorrect conclusions about the fairness of the model. In particular, rows1-4 in the table shows that counterexamples were not found in the neighborhood of the training data (reflected by the $100\%$ certified local fairness), yet verifying the product NN was capable of finding counterexamples that are far from the training data leading to accurate conclusions about the NN fairness.


\begin{figure}[!t]
    \centering
    \begin{subfigure}{0.49\textwidth}
    \includegraphics[width=\textwidth,scale=0.5]{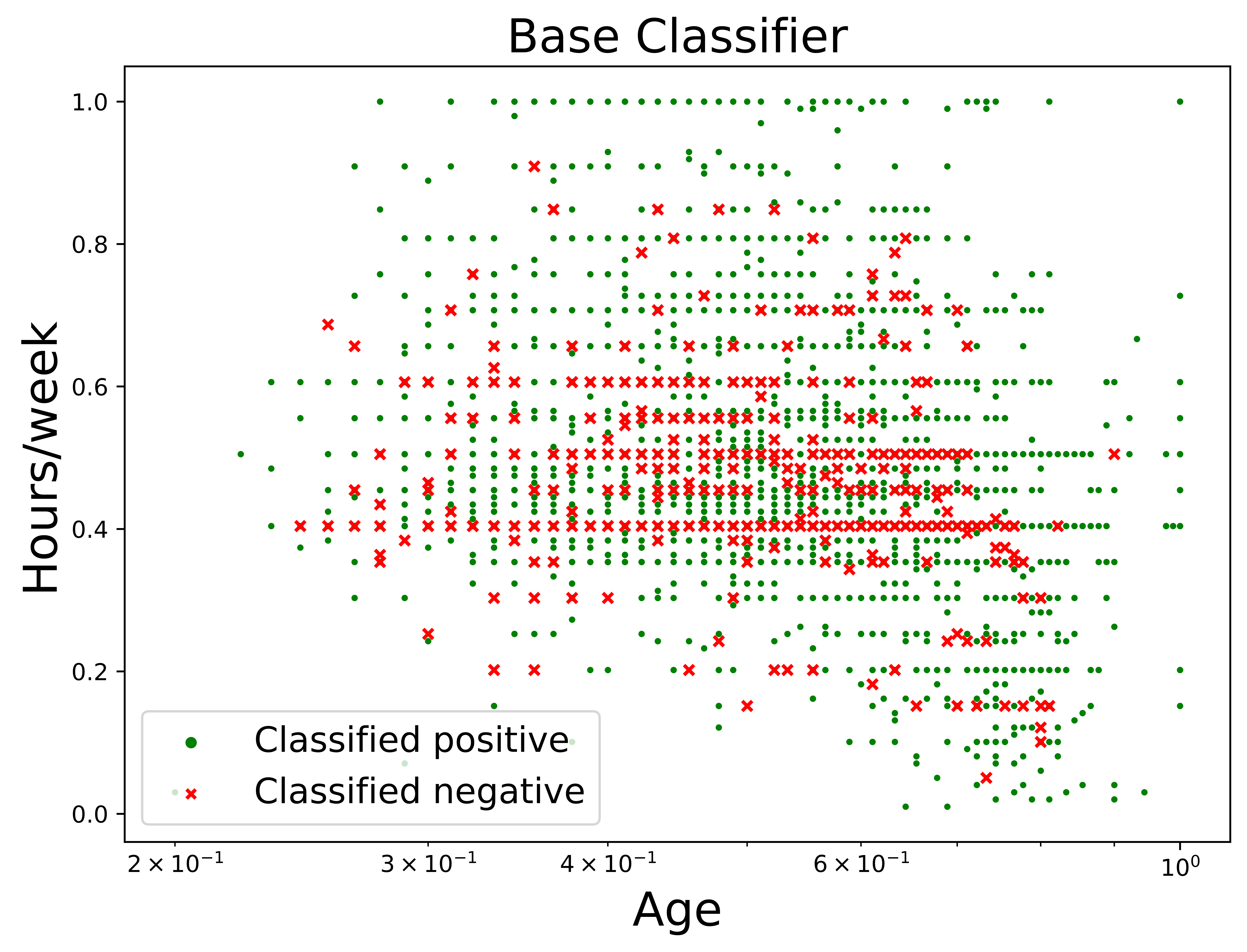}
    \end{subfigure}
    \begin{subfigure}{0.49\textwidth}
    \includegraphics[width=\textwidth,scale=0.5]{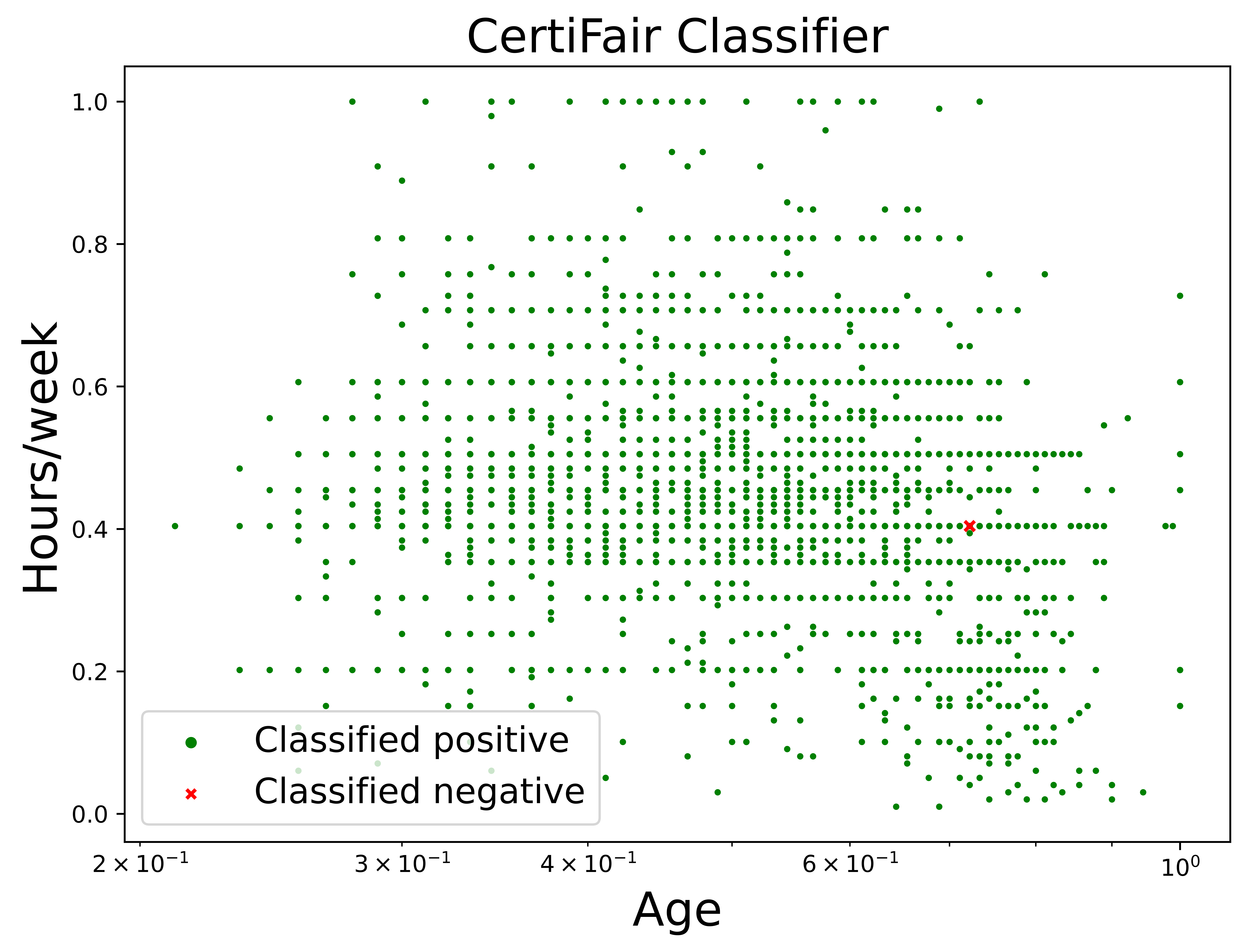}
    \end{subfigure}
    \vspace{-3mm}
    \caption{Comparison between the base and \toolname\ classifiers in terms of fairness as defined in \ref{def:counterfactual}. We show the classifications for the minority group of the adult dataset projected on two features; Age, and hours worked per week. The figure shows that the base classifier suffers from biases against identical individuals who are of different race (red markers). \toolname\ is able to drastically improve the individual fairness on this dataset with only $2\%$ reduction in accuracy.}
    \label{fig:discrimination} \vspace{-5mm}
\end{figure}

\begin{figure}[!t]
    \centering
    
    \includegraphics[width=\textwidth]{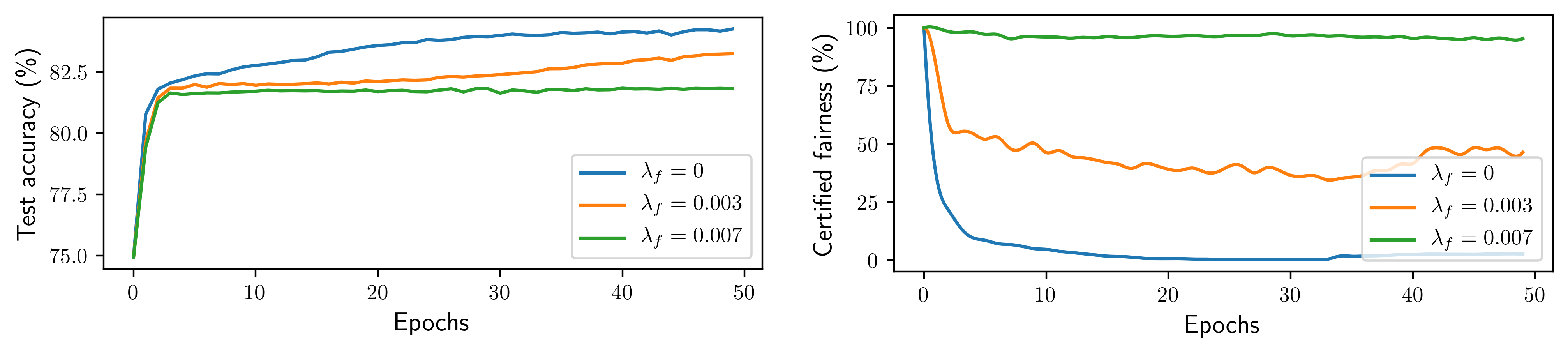} \vspace{-7mm}
    \caption{Test accuracy (left) and certified fairness (right) across training epochs when training a NN with the fairness constraint ($\lambda_f = 0.007$) and without it ($\lambda_f =0)$ on the Adult dataset.}
    \label{fig:training_curves} \vspace{-5mm}
\end{figure}

\subsection{Experiment 2: Effects of Incorporating the Fairness Regularizer 
}
We investigate the effect of using the global fairness regularizer (defined in~\eqref{eq:reg_global_2}) on the decisions of the NN classifier when trained on the Adult dataset. The fairness property for this experiment is of class $\mathcal{P}_1$. To investigate the predictor's bias, we first project the data points on two numerical features (age and hours/week). Our objective is to check whether the points that are classified positively for the privileged group are also classified positively for the non privileged group. Figure~\ref{fig:discrimination} (left) shows the predictions for the unprivileged group when using the base classifier ($\lambda_f = 0$). Green markers indicate points for which individuals from both privileged and non privileged groups are treated equally while red markers show individuals from the non privileged group that did not receive the same NN output compared to the corresponding identical ones in the privileged group. Figure \ref{fig:discrimination} (right) shows the same predictions but using the fair classifier ($\lambda_f = 0.03$), the predictions from this classifier drastically decreased the discrimination between the two groups while only decreasing the accuracy by 2\%. These results suggest that we can indeed regularize the training to improve the satisfaction of the fairness constraint without a drastic change in performance.


We also investigate how the certified fairness changes across epochs of training. To that end, we train a NN for the German dataset and 
evaluate the test accuracy as well as the certified global fairness after each epoch of training for two different values of $\lambda_f$. Figure \ref{fig:training_curves} shows the underlying trade-off between achieving fairness versus maximizing accuracy. As expected, lower values of $\lambda_f$ results in lower loss in accuracy (compared to the base case with $\lambda_f = 0$) while having lower effect on the fairness. The results also show that a small sacrifice of the accuracy can lead to significant enhancement of the fairness as shown for the $\lambda_f = 0.007$ case.

\subsection{Experiment 3: Certified Fair Training}
\label{subsec:exp3}

\textbf{Experiment setup:} 
The objective of this experiment is to compare the two regularizers, the local fairness regularizer in~\eqref{eq:reg_local_2} and the global fairness regularizer in~\eqref{eq:reg_global_2}. To that end, we performed a grid search over the learning rate $\alpha$, the fairness regularization parameter $\lambda_f$, and the NN architecture to get the best test accuracy across all datasets. Best performance was obtained with a NN that consists of two hidden layers of 20 neurons (except for the German dataset, where we use 30 neurons per layer), learning rate $\alpha=0.001$, global fairness regularization parameter $\lambda_f$ equal to 0.01 for Adult and Law School, 0.005 for German, and 0.1 for Compas dataset, and local fairness regularization parameter $\lambda_f$ equal to 0.95 for Adult, 0.9 for Compas, 0.2 for German, and 0.5 for Law School. We trained the models for 50 epochs with a batch size of 256 (except for Law School, where the batch size is set to 1024) and used Adam Optimizer for learning the NN parameters $\theta$. All the datasets were split into $70\%$ and $30\%$ for training and testing, respectively.

\begin{table}[]
\caption{Comparison between a base classifier ($\lambda_f=0$) and \toolname\ classifier with different fairness regularizers}
\label{tab:exp3}
\resizebox{\textwidth}{!}{%
\begin{tabular}{cllllllllll}
\toprule
Constraint &
  \multicolumn{1}{c}{Dataset} &
  \multicolumn{3}{c}{Test Accuracy (\%)} &
  \multicolumn{3}{c}{Positivity Rate(\%)} &
  \multicolumn{3}{c}{\begin{tabular}[c]{@{}c@{}}Certified \\ Global Fairness (\%)\end{tabular}} \\
\multicolumn{2}{l}{} &
  Base &
  \multicolumn{1}{c}{\begin{tabular}[c]{@{}c@{}}\ $\mathcal{L}_f^g$\end{tabular}} &
  \multicolumn{1}{c}{\begin{tabular}[c]{@{}c@{}}\ $\mathcal{L}_f^l$\end{tabular}} &
  \multicolumn{1}{c}{Base} &
  \multicolumn{1}{c}{\begin{tabular}[c]{@{}c@{}}\ $\mathcal{L}_f^g$\end{tabular}} &
  \multicolumn{1}{c}{\begin{tabular}[c]{@{}c@{}}\ $\mathcal{L}_f^l$\end{tabular}} &
  \multicolumn{1}{c}{Base} &
  \multicolumn{1}{c}{\begin{tabular}[c]{@{}c@{}}\ $\mathcal{L}_f^g$\end{tabular}} &
  \multicolumn{1}{c}{\begin{tabular}[c]{@{}c@{}}\ $\mathcal{L}_f^l$\end{tabular}} \\ \hline
\multirow{4}{*}{
$\mathcal{P}_2$} &
  Adult &
  84.55 &
  82.34 &
  83.81 & 
  20.8 &
  17.4 &
  21.79&
  6.40 &
  100.00 &
  61.92
   \\
 &
  German & 75.30	&73.00	& 70.00  &79.00	&72.00	& 83.00  &8.64	&95.06 &86.41  
   \\
 &
  Compas & 68.30  & 65.08 & 63.19  & 61.55 & 66.00      &69.40 & 47.22 & 100.00 & 100.00
   \\
 &
  Law & 87.60  & 79.92 & 78.69 & 21.51 & 9.10   & 25.03 & 6.87  & 51.87 & 78.54 
   \\ \hline
\multirow{4}{*}{
$\mathcal{P}_1$
} &
  Adult & 84.55 & 82.33 &  83.25 & 20.80  & 18.13 &  20.15   &    1.77   & 97.86& 95.31
   \\
 &
  German &
  75.30 &  72.66  &  69.66 & 79.00    & 83.14  & 81.71  & 14.81 & 92.59 &82.71
   \\
 &
  Compas &68.30  & 65.08 &  63.82 & 61.55 & 66.00    &   71.60    & 47.22 & 100.00 &100.00
   \\
 &
  Law & 87.60	&84.90	&86.69   &21.42	&17.50	&  21.63 &34.16	&70.10 & 86.45 \\ \bottomrule
  
\end{tabular}%
}
\end{table}

\textbf{Effect of the choice of the fairness regularizer:}
We investigate the certified global fairness for the two regularizers introduced in Section \ref{subsec:fair_training}. Table \ref{tab:exp3} summarizes the results for $\mathcal{P}_1$ and $\mathcal{P}_2$ fairness properties across different datasets. For each property and dataset, we compare the test accuracy, positivity rate (percentage of points classified as 1), and the certified global fairness of the base classifier (trained with $\lambda_f = 0$) and the \toolname\ classifier trained twice with two different fairness regularizers $\mathcal{L}_f^l$ and $\mathcal{L}_f^g$. Compared to the base classifier, training the NN with the global fairness regularizer $\mathcal{L}_f^g$ significantly increases the certified global fairness with a small drop in the accuracy in most of the cases except for the Law School dataset, where the test accuracy dropped by 7 \% on $\mathcal{P}_2$ but the global fairness increased by 55 \%. Compared to the local regularizer $\mathcal{L}_f^l$, the global regularizer achieves higher global fairness and comparable (if not better) test accuracy on all datasets except Law School. We think that this might be due to the network's limited capacity to optimize both objectives. We also report the positivity rate (number of data points classified positively) for the classifiers. This metric is important because most of these datasets are unbalanced, and hence the classifiers can trivially skew all the classifications to a single label and achieve high fairness percentage. Thus it is desired that the positivity rate of the \toolname\ classifier to be close to the one of the base classifier to ensure that it is not trivial. Lastly, we conclude that even though the local regularizer improves the global fairness, the global regularizer can achieve higher degrees of certified global fairness without a significant decrease in test accuracy, and of course, it avoids the drawbacks of the local regularizer discussed in Section \ref{subsec:fair_training}.

\textbf{Effect of the fairness regularization parameter:}
In this experiment, we investigate the effect of the fairness regularization parameter $\lambda_f$ on the classifier's accuracy and fairness. The parameter $\lambda_f$ controls the trade-off between the accuracy of the classifier and its fairness, and tuning this parameter is usually dependent on the network/dataset. To that end, we trained a two-layer NN with 30 neurons per layer for the German dataset using 8 different values for $\lambda_f$ and summarized the results in Table \ref{tab:exp2}. The fairness property verified is of class $\mathcal{P}_2$. The results show that the global fairness satisfaction can increase without a significant drop in accuracy up to a certain point, after which the fairness loss is dominant and results in a significant decrease in the classifier's accuracy.

\begin{table}[]
\caption{Effect of fairness regularization parameter $\lambda_f$ on the test accuracy and certified fairness.}
\label{tab:exp2}
\resizebox{\textwidth}{!}{%
\begin{tabular}{cllllllll}
\toprule
$\lambda_f$   & $1 \times 10^{-4}$ & $5\times 10^{-4}$ & $7\times 10^{-4}$ & $5\times 10^{-3}$ & $7\times 10^{-3}$ & $1 \times 10^{-2}$  & $2\times 10^{-2}$  & $5\times 10^{-2}$  \\ \hline
Test accuracy (\%)                 & 74.33    & 74.33    & 75.00       & 72.33    & 72.66    & 73.00    & 72.6  & 66.33 \\
Certified global fairness  (\%)  & 8.64     & 14.81    & 13.58    & 66.66    & 82.71    & 95.06 & 98.76 & 100  \\ \bottomrule
\end{tabular}%
}
\end{table}

\section{Related work}

\textbf{Group fairness:} Group fairness considers notions like demographic parity \cite{feldman2015certifying}, equality of odds, and equality of opportunity \cite{hardt2016equality}. Tools that verify notions of group fairness assume knowledge of a probabilistic model of the population. FairSquare \cite{aws2017fairsquare} relies on numerical integration to formally verify notions of group fairness; however, it does not scale well for NNs. VeriFair \cite{bastani2019probabilistic} considers probabilistic verification using sampling and provides soundness guarantees using concentration inequalities. This approach is scalable to big networks, but it does not provide worst-case proof. 

\textbf{Individual fairness:} More related to our work is the verification of individual fairness properties. LCIFR \cite{ruoss2020learning} proposes a technique to learn fair representations that are provably certified for a given individual. An encoder is \emph{empirically} trained to map similar individuals to be within the neighborhood of the given individual and then apply NN verification techniques to this neighborhood to certify fairness. The property verified is a \emph{local} property with respect to the given individual. On the contrary, our work focuses on the global fairness properties of a NN. It also avoids the empirical training of similarity maps to avoid affecting the soundness and completeness of the proposed framework. In the context of individual global fairness, a recent work \cite{john2020verifying} proposed a sound but incomplete verifier for linear and kernelized polynomial/radial basis function classifiers. It also proposed a meta-algorithm for the global individual fairness verification problem; however, it is not clear how it can be used to design sound and complete NN verifiers for the fairness properties. 
Another line of work \cite{urban2020perfectly} focuses on proving \textit{dependency} fairness properties which is a more restrictive definition of fairness since it requires the NN outputs to avoid any dependence on the sensitive attributes. The method employs forward and backward static analysis and input space partitioning to verify the fairness property. As mentioned, this definition of fairness is different from the individual fairness we are considering in this work and is more restrictive.

\textbf{NN verification:} This work is algorithmically related to NN verification in the context of adversarial robustness. However, adversarial robustness is a local property of the network given a nominal input, and a norm bounded perturbation. Moreover, the robustness property does not consider the notion of sensitive attributes. The NN verification literature is extensive, and the approaches can be grouped into three main groups: (i) SMT-based methods, which encode the problem into a 
Satisfiability Modulo Theory problem ~\cite{katz2019marabou, 
KatzReluplexEfficientSMT2017a, ehlers2017formal}; (ii) MILP-based solvers, 
which solves the verification problem exactly by encoding it as a Mixed Integer Linear 
Program ~\cite{lomuscio2017approach, tjeng2017evaluating, bastani2016measuring, bunel2020branch, fischetti2018deep, anderson2020strong, cheng2017maximum}; (iii) Reachability based methods ~\cite{xiang2017reachable, xiang2018output, gehr2018ai2, wang2018formal, tran2020nnv, ivanov2019verisig, fazlyab2019efficient}, which perform layer-by-layer reachability analysis to compute a reachable set that can be verified against the property; and (iv) convex 
relaxations methods ~\cite{wang2018efficient, dvijotham2018dual, wong2017provable,henriksen2021deepsplit,khedr2021peregrinn,wang2021beta}. Generally, (i), (ii), and (iii) do not scale well to large networks. On the other hand, convex relaxation methods use a branch and bound approach to refine the abstraction.


\section{Discussion:}
\label{sec:discussion}
\paragraph{On the contention between Group and Individual fairness:} 
Group fairness is the requirement that different groups should be treated similarly regardless of individual merits. It is often thought of as a contradictory requirement to individual fairness. However, this has been an issue of debate \cite{binns2020apparent}. Thus, it's not clear how our framework for training might affect the group fairness requirement and is left for further investigation.

\paragraph{Can fairness be achieved by dropping sensitive attributes from data?} Fairness through unawareness is the process of learning a predictor that does not explicitly use sensitive attributes in the prediction process. However, dropping the sensitive attributes is not sufficient to remove discrimination as it can be highly predictable implicitly from other features. It has been shown ~\cite{bonilla2006racism, taslitz2007racial, pedreshi2008,amazon} that discrimination highly occurs in different systems such as housing, criminal justice, and education that do not explicitly depend on sensitive attributes in their predictions.



\bibliographystyle{ieeetr}
\bibliography{mybib}

\begin{thebibliography}{10}
\bibitem{amazon}
Amazon, ``Amazon doesn’t consider the race of its customers. should it?.''
\end{thebibliography}

\clearpage

\clearpage


\end{document}